\documentclass{article}
\usepackage{amsmath, amssymb, graphicx}
\usepackage{booktabs}
\usepackage{amsthm} 
\usepackage{hyperref}
\usepackage{algorithm}
\usepackage{algpseudocode}
\usepackage[numbers]{natbib}
\usepackage{booktabs} % \toprule, \midrule, \bottomrule
\usepackage{array}    % \newcolumntype
\usepackage{caption}
\usepackage[a4paper, total={6in, 8in}, left=0.5in, right=0.5in, top=1in, bottom=1in]{geometry} 
\usepackage{algorithm}
\usepackage{algorithmicx}
\usepackage{algpseudocode}
\usepackage{color}

\usepackage[below]{placeins}
\usepackage{sidecap}
\usepackage{graphicx}
\usepackage{moreverb}   % \verbatimtabinpu
\usepackage{hyperref}
\usepackage{fancyhdr}
\usepackage[latin1]{inputenc}
\usepackage{wrapfig}
\usepackage[numbers]{natbib}
\newtheorem{thm}[subsection]{Theorem}

\newtheorem{conj}[subsection]{Conjecture}
\newtheorem{pro}[subsection]{Proposition}
\newtheorem{exa}[subsection]{Example}
\newtheorem{defn}[subsection]{Definition}

\numberwithin{equation}{section}

\title{Bellman operator convergence enhancements in reinforcement learning algorithms}
\author{David Krame Kadurha\thanks{david.krame@aims-cameroon.org} \\  
        University of Edinburgh, \\ UK \\
        \and
        Domini Jocema Leko Moutouo\thanks{domini@aims.ac.za} \\ 
        AIMS Cameroon, \\ 
        African Institute for Mathematical Sciences
               \and
        Ya\'e Ulrich Gaba\thanks{yaeulrich.gaba@gmail.com} \\ 
        Department of Mathematics and Applied Mathematics \\ Sefako Makgatho Health Sciences University (SMU)\\ Pretoria, South Africa
}

\begin{document}

\maketitle
\begin{abstract}
    %This paper explores the topological foundations of reinforcement learning (RL) by focusing on state, action, and policy spaces in RL. We investigate how the Banach fixed-point theorem explains the convergence of RL algorithms and propose new approaches to enhance their efficiency. These contributions are validated through experimental results on standard RL environments such as MountainCar, CartPole, and Acrobot. Our findings show that alternative formulations of Bellman operators can improve convergence rates and algorithmic performance.
This paper reviews the topological groundwork for the study of reinforcement learning (RL) by focusing on the structure of state, action, and policy spaces. We begin by recalling key mathematical concepts such as complete metric spaces, which form the foundation for expressing RL problems. By leveraging the Banach contraction principle, we illustrate how the Banach fixed-point theorem explains the convergence of RL algorithms and how Bellman operators, expressed as operators on Banach spaces, ensure this convergence.  The work serves as a bridge between theoretical mathematics and practical algorithm design, offering new approaches to enhance the efficiency of RL. In particular, we investigate alternative formulations of Bellman operators and demonstrate their impact on improving convergence rates and performance in standard RL environments such as MountainCar, CartPole, and Acrobot. Our findings highlight how a deeper mathematical understanding of RL can lead to more effective algorithms for decision-making problems.
    
\end{abstract}

\section{Introduction}

Research on the foundational aspects of Reinforcement Learning (RL), particularly from a topological perspective, remains relatively underdeveloped. While RL has seen significant advancements in terms of algorithmic efficiency and practical applications, there is a lack of comprehensive studies that address the underlying mathematical structures of the problem spaces (namely, state, action, and policy spaces). This work aims to consolidate and formalize these fundamental RL concepts by grounding them in a coherent mathematical framework, with a particular focus on topological and geometric perspectives. In existing literature, several contributions have touched on aspects of RL that relate to topology and geometry, but a unified approach is still missing. For instance, \cite{ValuePolytope} explores the geometric and topological properties of value functions within Markov decision processes (MDPs) that have finite states and actions. This work characterizes the value function space as a polytope, elucidating the intricate relationships between policies and value functions. Similarly, \cite{DBLP:journals/corr/abs-1804-07193} highlights the importance of Lipschitz continuity in model-based RL, advocating for the learning of Lipschitz continuous models to improve value function estimation and providing theoretical error bounds. The study in \cite{DBLP:journals/corr/abs-1905-00475} expands the scope of model-free RL algorithms to continuous state-action spaces using a Q-learning-based approach, thereby extending the applicability of RL to more complex domains. Additionally, \cite{MetricAndContinuityInRL} introduces a unified framework for defining state similarity metrics in RL, addressing the generalization challenges posed by continuous-state systems. This framework offers new insights into how metric spaces can be leveraged to reason about the learning process in RL. Despite these contributions, the lack of a cohesive foundation to interpret these insights and relate them back to a topological understanding of RL limits the impact of such works. Our primary goal is to address this gap by establishing a solid foundation that unifies and organizes the key concepts in RL, particularly in relation to the topology of state, action, and policy spaces. To this end, we build upon mathematical constructs such as metric spaces, normed spaces, and Banach spaces, and we explore how the Banach fixed-point theorem, specifically through the Banach contraction principle, can be applied to explain the convergence of RL algorithms. Furthermore, we propose alternative formulations of Bellman operators, expressed as operators on Banach spaces, to demonstrate how such theoretical insights can lead to improved convergence rates and more efficient algorithmic performance. These contributions are validated through experimental results on standard RL environments such as MountainCar, CartPole, and Acrobot, showing that our approach not only strengthens the theoretical understanding of RL but also offers practical benefits. By laying the groundwork for studying the topology of RL problem spaces, this work aims to serve as a foundational reference for future researchers. We anticipate that the mathematical insights and frameworks presented will aid in the development of more effective algorithms, ultimately advancing the field of reinforcement learning by bridging theory and practice.

\vspace{0.5cm}

\noindent
The remainder of this paper is structured as follows, and represents a polished and refined version of the earlier work presented by us in \cite{kadurha2024topological}. In Section \ref{section2}, we begin by introducing the foundational mathematical concepts needed for our analysis. Specifically, in Section 2.1, we review contraction mappings and fixed-point theorems, including the Banach contraction principle, which underpins much of reinforcement learning's convergence theory. In Section 2.2, we provide an overview of reinforcement learning, framing it within the context of Markov Decision Processes and setting the stage for the subsequent mathematical discussions. In Section \ref{section3}, we delve deeper into the theoretical aspects of reinforcement learning, starting with a refinement of the Banach contraction principle in Section 3.1, followed by an examination of Bellman optimality operators in Section 3.2. Section 3.3 explores policy evaluation and iteration through the lens of operator theory, emphasizing the importance of these mathematical structures in understanding RL algorithms. Section \ref{section4} introduces alternative formulations to the classical Bellman operator. Section 4.1 discusses the motivation behind seeking such alternatives, while Sections 4.2 and 4.3 present the Consistent Bellman Operator and the Modified Robust Stochastic Operator, respectively. These alternatives are proposed to address some of the limitations of the classical operator, offering improved convergence properties and robustness in various settings. Finally, in Section \ref{section5}, we provide detailed implementations of the proposed concepts and analyze their performance through experimental results on standard RL environments. This section demonstrates the practical impact of our theoretical findings and validates the proposed approaches.

\section{Preliminaries}\label{section2}

\subsection{Contraction mappings and fixed points}

In this section, we revisit the concept of contraction mappings and the Banach fixed-point theorem, which is foundational for understanding the convergence properties of many algorithms, including those in reinforcement learning (RL). We begin by recalling the notion of a contraction mapping and then provide the Banach fixed-point theorem along with its proof. We also discuss the relevance of these mathematical results to RL, particularly in the context of value iteration and Bellman operators.

%\subsection{Contraction mappings and fixed points}
\begin{defn}
	Let $X$ be a nonempty set and $f: X \to X$ be a mapping on that set.
	\begin{itemize}
		\item[-] A point $x$ is said to be a \textbf{fixed point} of $f$ if $f(x)=x$.
		\item[-] We will write $Fix(f) = \{x\in X : f(x) = x\}$, the set of fixed points of $f$ on $X$.
	\end{itemize}
\end{defn}

\begin{pro}\label{prop : propFix}
	Let $X$ be a nonempty set and $f:X\to X$ a mapping defined on it. If $x\in X$ is a unique fixed point of $f^n$ with $f^n = \underset{n-times}{\underbrace{f\circ f\circ\cdots\circ f}}$ for any $n>1$, then it is the unique fixed point of $f$ and vice versa:
	\[Fix(f^n) = \{x\} \iff Fix(f) = \{x\}.\]
\end{pro}

%\begin{defn}\cite{debnath2021metric, gopal2017background,naylor1982linear,pata2019fixed}
%	Let $(X,d)$ be a metric space and $f:X\to X$ a mapping on $X$. $f$ is said to be \textit{Lipschitz} continuous if there exists $\alpha>0$ such that:
%	\begin{equation}
%		d\left(f(x),f(y)\right)\leq \alpha\cdot d(x,y)~~~ \forall x,y \in X.
%	\end{equation}
%	\begin{itemize}
%		\item[-] If $\alpha\in [0,1)$, $f$ is said to be a \textit{contraction}.
%		\item[-] If $\alpha=1$, then $f$ is said to be \textit{non-expansive}.
%		\item[-] If $d\left(f(x),f(y)\right)<d(x,y), ~ \forall x\neq y$, then $f$ is contractive.
%	\end{itemize}
%\end{defn}

\subsection*{Contraction mappings}

Let \( (X, d) \) be a \textit{metric space}, where \( d : X \times X \to \mathbb{R} \) is a distance function that satisfies the usual properties of a metric: non-negativity, identity of indiscernibles, symmetry, and the triangle inequality. A mapping \( T : X \to X \) is called a \textit{contraction mapping} if there exists a constant \( \alpha \in [0, 1) \) such that, for all \( x, y \in X \),
\[
d(T(x), T(y)) \leq \alpha d(x, y).
\]
The constant \( \alpha \) is called the \textit{contraction constant}. The key idea is that a contraction mapping brings points closer together, ensuring that successive applications of \( T \) shrink distances between any two points.

\vspace{0.3cm}

\noindent
The first interesting result in this context is the following:

\begin{pro}\label{pro1}
	Let $(X,d)$ be a metric space and $f:X\to X$ a contraction mapping with $\alpha \in (0,1)$. If $f$ has a fixed point, that point is unique.
\end{pro}
\begin{proof}
	Suppose we have two fixed points $x$ and $y$ for $f$ with $x\neq y$. Because $f$ is a contraction, we can write:
	\[0\neq d(x,y) = d\left(f(x),f(y)\right)\leq \alpha\cdot d(x,y),\]
	which is a contradiction. Thus, the fixed point is unique.
\end{proof}

\subsection*{The Banach contraction principle : BCP }

A powerful result concerning contraction mappings on complete metric spaces is the \textit{Banach fixed-point theorem} (also known as the contraction mapping theorem), which guarantees the existence and uniqueness of a fixed point for such mappings. A \textit{fixed point} \( x^* \in X \) is a point that satisfies \( T(x^*) = x^* \).

\begin{thm}[See \cite{fixedpoint}]\label{thm:BCP}
%[Banach Fixed-Point Theorem]
Let \( (X, d) \) be a \textit{complete metric space} and let \( T : X \to X \) be a contraction mapping with contraction constant \( \alpha \in [0, 1) \). Then, the following holds:
\begin{enumerate}
    \item \( T \) has a unique fixed point \( x^* \in X \), i.e., there exists a unique \( x^* \in X \) such that \( T(x^*) = x^* \).
    \item For any initial point \( x_0 \in X \), the sequence \( \{x_n\} \) defined by \( x_{n+1} = T(x_n) \) converges to \( x^* \) as \( n \to \infty \). Moreover, the convergence is geometric, i.e., 
    \[
    d(x_n, x^*) \leq \frac{\alpha^n}{1 - \alpha} d(x_0, x_1).
    \]
\end{enumerate}
\end{thm}

%\subsection*{Proof of the Banach Fixed-Point Theorem}

\vspace{0.3cm}

\noindent
The proof of the Banach fixed-point theorem is constructive and proceeds as follows:

\begin{proof}
Let \( x_0 \in X \) be an arbitrary initial point. Define a sequence \( \{x_n\} \) by \( x_{n+1} = T(x_n) \) for all \( n \geq 0 \). We will first show that \( \{x_n\} \) is a Cauchy sequence and then that it converges to a unique limit.

\paragraph{Step 1: Show that \( \{x_n\} \) is Cauchy.}

For \( n \geq 0 \), we have:
\[
d(x_{n+1}, x_n) = d(T(x_n), T(x_{n-1})) \leq \alpha d(x_n, x_{n-1}),
\]
where the inequality follows from the contraction property of \( T \). By applying this recursively, we get:
\[
d(x_{n+1}, x_n) \leq \alpha^n d(x_1, x_0).
\]
Summing this geometric series, we find that for \( m > n \),
\[
d(x_m, x_n) \leq \sum_{k=n}^{m-1} d(x_{k+1}, x_k) \leq \sum_{k=n}^{m-1} \alpha^k d(x_1, x_0) \leq \frac{\alpha^n}{1 - \alpha} d(x_1, x_0).
\]
Since \( \alpha \in [0, 1) \), this bound tends to zero as \( n \to \infty \), which shows that \( \{x_n\} \) is a Cauchy sequence.

\paragraph{Step 2: Show that \( \{x_n\} \) converges.}

Since \( X \) is complete, every Cauchy sequence in \( X \) converges. Therefore, there exists a point \( x^* \in X \) such that \( x_n \to x^* \) as \( n \to \infty \).

\paragraph{Step 3: Show that \( x^* \) is a fixed point.}

We now show that \( T(x^*) = x^* \). Since \( T \) is continuous and \( x_n \to x^* \), we have:
\[
T(x_n) \to T(x^*) \quad \text{as} \quad n \to \infty.
\]
However, by construction, \( T(x_n) = x_{n+1} \), and since \( x_n \to x^* \), it follows that \( x_{n+1} \to x^* \). Therefore, \( T(x^*) = x^* \).

\paragraph{Step 4: Show uniqueness.}

The uniqueness follows from Proposition \ref{pro1}.
%Finally, we show that \( x^* \) is the unique fixed point. Suppose there is another fixed point \( y^* \) such that \( T(y^*) = y^* \). Then, using the contraction property, we have:
%\[d(x^*, y^*) = d(T(x^*), T(y^*)) \leq \alpha d(x^*, y^*).\]
%Since \( \alpha \in [0, 1) \), the only solution to this inequality is \( d(x^*, y^*) = 0 \), which implies \( x^* = y^* \). Thus, the fixed point is unique.
\end{proof}

%\subsection{Relevance to Reinforcement Learning}

\noindent
The Banach fixed-point theorem is central to the analysis of many reinforcement learning algorithms, particularly those involving value function approximation. In the context of reinforcement learning, the Bellman operator, often used in value iteration, is a contraction under the supremum norm. Specifically, for a Markov Decision Process (MDP), the Bellman optimality operator \( T^* \) satisfies the contraction property with a contraction constant \( \gamma \), where \( \gamma \) is the discount factor.  Thus, by applying the Banach fixed-point theorem, we can guarantee the existence and uniqueness of an optimal value function \( V^* \), and the iterative application of the Bellman operator ensures convergence to this fixed point. This theoretical result forms the basis for the convergence of foundational RL algorithms, such as value iteration and policy iteration.

%Contraction mappings and the Banach fixed-point theorem provide the theoretical foundation for ensuring the convergence of iterative algorithms. In reinforcement learning, the Bellman operator can be viewed as a contraction mapping, guaranteeing convergence to the optimal value function. The formal understanding of these concepts is crucial for analyzing the performance and convergence of RL algorithms, and they will be further explored in the subsequent sections in the context of operator-based approaches to RL.

\subsection{Overview on Reinforcement Learning}\label{chap2}

In this part, we will formally present the framework of a Markov Decision Process (MDP) and its relation to Reinforcement Learning (RL). Following this, we will discuss the concept of optimality in RL, and finally, outline key methods for achieving optimality in decision-making tasks.

\begin{defn}[Markov Decision Process \cite{lazaric2013markov, sigaud2013markov}]
	A Markov Decision Process (MDP) is defined as a 5-tuple \( \mathcal{M} = \langle \mathcal{S}, \mathcal{A}, p, r, \gamma \rangle \), where:
	\begin{itemize}
		\item \( \mathcal{S} \): \textbf{State space}, the set of all possible states.
		\item \( \mathcal{A} \): \textbf{Action space}, the set of all possible actions the agent can take.
		\item \( p(s, a, s') \): \textbf{Transition probability}, which describes the probability of moving from state \( s \) to state \( s' \) given action \( a \):
		\[
		p(s'|s, a) = Pr(S_{t+1} = s' | S_t = s, A_t = a).
		\]
		\item \( r(s, a, s') \): \textbf{Reward function}, which defines the immediate reward received after transitioning from state \( s \) to state \( s' \) via action \( a \):
		\[
		r : \mathcal{S} \times \mathcal{A} \times \mathcal{S} \to \mathcal{R}\subset \mathbb{R}.
		\]
        We should also mention that sometimes it is convenient to simply use \( r(s,a) \), defined as follows:
        \[
		r(s,a) = \sum_{r\in\mathcal{R}}r\cdot\sum_{s'} p(s',r|s, a)~~~\text{where}~~~p(s',r|s, a) \equiv Pr(S_{t+1} = s', R_{t+1}=r | S_t = s, A_t = a)
		\]
		\item \( \gamma \in [0, 1) \): \textbf{Discount factor}, which determines the present value of future rewards, with smaller \( \gamma \) values giving more emphasis to immediate rewards.
	\end{itemize}
\end{defn}

\begin{defn}[Policy or Decision Rule]
	A policy \( \pi : \mathcal{S} \to \Delta \mathcal{A} \) (precisely from the set of states to the probability simplex under actions) defines a strategy that the agent uses to select actions. It can be either:
	\begin{itemize}
		\item \textbf{Deterministic}, where a specific action is chosen in each state, or
		\item \textbf{Stochastic}, where actions are chosen according to a probability distribution over actions in each state.
	\end{itemize}
\end{defn}

\noindent
The agent interacts with the environment by observing sequences of states, taking actions, and receiving rewards. This interaction can be described as:
\[
S_0, A_0, R_1, S_1, A_1, R_2, S_2, A_2, R_3, \ldots
\]
The goal of the agent is to maximize the cumulative reward over time, referred to as the \textit{return}, which is formalized as:

\begin{equation} \label{equation: return-def}
G_t = R_{t+1} + \gamma R_{t+2} + \gamma^2 R_{t+3} + \cdots = \sum_{k=0}^{\infty} \gamma^k R_{t+k+1}, \quad \gamma \in [0, 1).
\end{equation}

\noindent
The value function \( v_{\pi}(s) \), representing the expected return starting from state \( s \) and following policy \( \pi \), is defined as:

\begin{equation} \label{equation: state-value-function}
v_{\pi}(s) = \mathbb{E}_{\pi} \left[ G_t \mid S_t = s \right] = \mathbb{E}_{\pi} \left[ \sum_{k=0}^{\infty} \gamma^k R_{t+k+1} \mid S_t = s \right].
\end{equation}

\noindent
Similarly, the action-value function \( q_{\pi}(s, a) \), representing the expected return starting from state \( s \), taking action \( a \), and then following policy \( \pi \), is defined as:

\begin{equation} \label{equation: action-value-function}
q_{\pi}(s, a) = \mathbb{E}_{\pi} \left[ G_t \mid S_t = s, A_t = a \right] = \mathbb{E}_{\pi} \left[ \sum_{k=0}^{\infty} \gamma^k R_{t+k+1} \mid S_t = s, A_t = a \right].
\end{equation}

\noindent
The optimal value function, \( v_{*}(s) \), and the optimal action-value function, \( q_{*}(s, a) \), are defined as the maximum value that can be obtained by any policy \( \pi \):

\begin{equation}
\begin{aligned}
v_{*}(s) &= \max_{\pi} v_{\pi}(s), \\
q_{*}(s, a) &= \max_{\pi} q_{\pi}(s, a).
\end{aligned}
\end{equation}

\subsection*{Bellman Optimality Equations}

The Bellman optimality equations describe the recursive relationship for the optimal value functions. For the state-value function \( v_{*}(s) \), the Bellman equation is:

\begin{equation} \label{equation: Bell-Optim-state-val-func}
v_{*}(s) = \max_{a} \sum_{s', r} p(s' \mid s, a) \left[ r + \gamma v_{*}(s') \right].
\end{equation}

For the action-value function \( q_{*}(s, a) \), the Bellman equation is:

\begin{equation} \label{equation: Bell-Optim-action-val-func}
q_{*}(s, a) = r(s, a) + \gamma \sum_{s'} p(s' \mid s, a) v_{*}(s').
\end{equation}

Solving an MDP means finding the optimal value functions \( v_{*}(s) \) or \( q_{*}(s, a) \), which leads to the determination of the optimal policy \( \pi^{*} \).

\subsection*{Solution Methods}

There are two main types of algorithms used to solve MDPs and achieve optimality:

\begin{itemize}
	\item \textbf{Model-based algorithms}: These algorithms rely on a known model of the environment, including the transition probabilities and reward function, to compute value functions and derive optimal policies.
	\item \textbf{Model-free algorithms}: These algorithms do not assume knowledge of the environment's model. Instead, they interact with the environment to estimate value functions and improve the policy through exploration and learning.
\end{itemize}

In model-free algorithms, a central challenge is balancing the exploration-exploitation trade-off: the agent must explore the environment sufficiently to discover the best actions, while also exploiting the knowledge gained to maximize rewards. Many techniques have been developed to tackle this challenge. But it remains an open question \cite{sutton2018reinforcement}.

\section{Bellman Operators and convergence of RL algorithms}\label{section3}

In this section, we frame reinforcement learning in terms of operators, which offers a more structured understanding of why RL algorithms efficiently converge to the optimal policy. We will refine the Banach contraction principle to suit our needs, define the Bellman operators, and introduce key methods for achieving optimality using this operator-based framework.

\subsection{Rephrasing of the Banach contraction principle}

%\paragraph{Contraction Mapping and Fixed Point}
Let $(X, ||\cdot||)$ be a normed space. A mapping $\mathcal{T} : X \to X$ is called a $\gamma$-contraction mapping if there exists a constant $\gamma \in [0,1)$ such that for any $x_1, x_2 \in X$:
\[
||\mathcal{T}x_1 - \mathcal{T}x_2|| \leq \gamma \cdot ||x_1 - x_2||.
\]
This inequality implies that the operator $\mathcal{T}$ shrinks distances between points by a factor of at least $\gamma$. Furthermore, if a sequence $\{x_n\} \subset X$ converges in norm to some $x \in X$, i.e.,
\[
x_n \underset{||\cdot||}{\longrightarrow} x,
\]
then the sequence $\{\mathcal{T}x_n\}$ will also converge to $\mathcal{T}x$. That is:
\[
\mathcal{T}x_n \underset{||\cdot||}{\longrightarrow} \mathcal{T}x.
\]
An element $x^* \in X$ is called a fixed point of $\mathcal{T}$ if:
\[
\mathcal{T}x^* = x^*.
\]
This fixed point is particularly important, as it represents the point where repeated applications of the operator leave the system unchanged, which aligns with the concept of reaching an optimal solution in reinforcement learning.

\begin{pro}[Refined Banach Contraction Principle]\label{thm:RefinedBCP}
	Let $(X, ||\cdot||)$ be a Banach space (a complete normed space), and let $\mathcal{T} : X \to X$ be a $\gamma$-contraction mapping. Then the following properties hold:
	\begin{enumerate}
		\item $\mathcal{T}$ has a unique fixed point $x^* \in X$.
		\item For any initial point $x_0 \in X$, the sequence defined by the iterative process $x_{n+1} = \mathcal{T}x_n$ converges to $x^*$ geometrically. Specifically, for all $n \geq 0$, we have:
		\[
		||x_n - x^*|| \leq \gamma^n ||x_0 - x^*||.
		\]
		This guarantees that the convergence rate is proportional to the contraction factor $\gamma$, meaning that the distance between the iterates and the fixed point decreases exponentially.
	\end{enumerate}
\end{pro}

\begin{proof}
    The existence and uniqueness of the fixed point follow directly from the Banach fixed-point theorem. To prove geometric convergence, we observe that since $\mathcal{T}$ is a contraction mapping, we have for all $n \geq 0$:
    \[
    ||x_{n+1} - x^*|| = ||\mathcal{T}x_n - \mathcal{T}x^*|| \leq \gamma \cdot ||x_n - x^*||.
    \]
    Applying this inequality recursively, we obtain:
    \[
    ||x_n - x^*|| \leq \gamma^n ||x_0 - x^*||,
    \]
    which tends to zero as $n \to \infty$, implying that $x_n$ converges to $x^*$ geometrically.
\end{proof}

%\paragraph{Relevance to Reinforcement Learning}

%The Banach contraction principle plays a crucial role in reinforcement learning, particularly in the convergence of iterative algorithms. In RL, the Bellman operator is typically a contraction under the appropriate norm (e.g., the supremum norm for value functions). This contraction property guarantees that, through repeated application of the Bellman operator, value iteration or other similar methods will converge to the optimal value function, corresponding to the optimal policy. The refined Banach contraction principle provides the theoretical guarantee that these iterative processes lead to a unique optimal solution.

\subsection{Bellman Optimality Operators}\label{Section3.3}
In reinforcement learning, operators are mappings in function spaces, and they provide a systematic approach to solving for the optimal value functions and policies. Let $\mathcal{M} = \langle \mathcal{S}, \mathcal{A}, p, r, \gamma \rangle$ be a Markov Decision Process (MDP), where:
\begin{itemize}
    \item $\mathcal{S}$ is the state space,
    \item $\mathcal{A}$ is the action space,
    \item $p(s'|s,a)$ is the transition probability,
    \item $r(s,a,s')$ is the reward function, and
    \item $\gamma \in [0,1)$ is the discount factor.
\end{itemize}
Let $\mathcal{V}$ be the space of bounded real-valued functions over $\mathcal{S}$, representing state-value functions, and let $\mathcal{Q}$ be the space of bounded real-valued functions over $\mathcal{S} \times \mathcal{A}$, representing action-value functions. We define the following operators:
\begin{itemize}
    % \item $\mathcal{V}$: the space of state-value functions,
    % \item $\mathcal{Q}$: the space of action-value functions,
    \item $\mathcal{T}^*_v : \mathcal{V} \rightarrow \mathcal{V}$: the Bellman Optimality Operator for state-value functions,
    \item $\mathcal{T}^*_Q : \mathcal{Q} \rightarrow \mathcal{Q}$: the Bellman Optimality Operator for action-value functions.
\end{itemize}

\begin{defn}
The Bellman Optimality equation for the state-value function is given by:
\[
v_*(s) = \underset{a}{\max}\left( r(s, a) + \gamma \cdot \sum_{s'} p(s' | s, a) \cdot v_*(s') \right).
\]
The Bellman Optimality Operator for state-value functions, denoted by $\mathcal{T}^*_v$, is defined as:
\begin{equation}\label{equation : Bellman-Operator}
    (\mathcal{T}^*_v f)(s) \equiv \underset{a}{\max} \left[ r(s, a) + \gamma \cdot \sum_{s'} p(s'|s, a) \cdot f(s') \right], \quad \forall f \in \mathcal{V}.
\end{equation}
\end{defn}

\paragraph{Properties of the Bellman Optimality Operator}
The Bellman Optimality Operator $\mathcal{T}^*_v$ has the following key properties \cite{deepminducl2021}:
\begin{itemize}
    \item \textbf{Contraction}: $\mathcal{T}^*_v$ is a $\gamma$-contraction, meaning:
    \[
    \|\mathcal{T}^*_v u - \mathcal{T}^*_v v\|_\infty \leq \gamma \cdot \|u - v\|_\infty, \quad \forall u, v \in \mathcal{V}.
    \]
    \item \textbf{Monotonicity}: $\mathcal{T}^*_v$ is monotonic, i.e.,
    \[
    u \leq v \implies \mathcal{T}^*_v u \leq \mathcal{T}^*_v v, \quad \forall u, v \in \mathcal{V}.
    \]
\end{itemize}

\begin{proof}
% \textcolor{red}{To prove contraction}:
% \[
% \left| \mathcal{T}^*_v u(s) - \mathcal{T}^*_v v(s) \right| \leq \gamma \cdot \|u - v\|_\infty.
% \]
% \textcolor{red}{For monotonicity}, assume $u(s) \leq v(s)$, then:
% \[
% \mathcal{T}^*_v u(s) - \mathcal{T}^*_v v(s) \leq 0 \implies \mathcal{T}^*_v u(s) \leq \mathcal{T}^*_v v(s).
% \]
Instead of proving this, we are going to sketch a proof for the action-value expectation operator for coherence with the remainder of this document, but the proofs are almost the same and use the same arguments (see the proof of Proposition \ref{ProofExpectAct-Val})
\end{proof}

By the Banach Contraction Principle, we can conclude:
\begin{itemize}
    \item The operator $\mathcal{T}^*_v$ has a unique fixed point $f^* \in \mathcal{V}$.
    \item For any starting point (function) $f_0$, the sequence defined by $f_{n+1} = \mathcal{T}^*_v f_n$ converges to $f^*$, which is then the optimal value function by monotonicity.
\end{itemize}

The Bellman Optimality Operator $\mathcal{T}^*_v$, as defined, allows us to compute the optimal value function, which corresponds to the optimal policy $\pi^*$.

\begin{defn}
Similarly, the Bellman Optimality equation for action-value functions is given by:
\[
q_*(s, a) = r(s, a) + \gamma \cdot \sum_{s'} p(s'|s, a) \cdot \underset{a'}{\max} q_*(s', a').
\]
The Bellman Optimality Operator for action-value functions, denoted by $\mathcal{T}^*_Q$, is then defined as:
\begin{equation}\label{equation : Bellman-Operator-for-Q}
    (\mathcal{T}^*_Q f)(s, a) \equiv r(s, a) + \gamma \cdot \sum_{s'} p(s'|s, a) \cdot \underset{a'}{\max} f(s', a').
\end{equation}
\end{defn}

\noindent
The contraction and monotonicity properties that hold for $\mathcal{T}_v$ also apply to $\mathcal{T}_Q$ using similar arguments. These properties extend to the Expectation Operator, whether it relates to the action-value function or the state-value function. Below is the Expectation Operator for the action-value function, which will be used later in this work :
\begin{equation}\label{eq:BellExpeOperQValues}
		\left(\mathcal{T}^\pi_Q f\right)(s,a) = r(s,a)+\gamma\cdot\sum_{s'}p(s'|s,a)\left(\sum_{a'}\pi(a'|s')\cdot f(s',a')\right)
	\end{equation}
\begin{pro}\label{ProofExpectAct-Val}
    We claim that the operator $\mathcal{T}^\pi_Q$, as defined in Equation \ref{eq:BellExpeOperQValues}, is:
    \begin{enumerate}
        \item A $\gamma$-contraction mapping.
        \item A monotonic mapping.
    \end{enumerate}
\end{pro}

\begin{proof}
    Let us consider two functions $u$ and $v$ in the space of action-value functions $\mathcal{Q}$.
    \begin{enumerate}
        \item \textbf{Contraction mapping:}

        Let us compute the absolute value of the difference between the transformations of $u$ and $v$ under $\mathcal{T}^\pi_Q$:
        \begin{align}
            \left|\mathcal{T}^\pi_Q u(s, a) - \mathcal{T}^\pi_Q v(s, a)\right| &= \gamma \cdot \left|\sum_{s'} p(s'|s, a) \left(\sum_{a'} \pi(a'|s') u(s', a') - \sum_{a'} \pi(a'|s') v(s', a')\right)\right| \nonumber \\
            % &= \gamma \cdot \left|\sum_{s'} p(s'|s, a) \left(\mathbb{E}_{a'|s'} u(s', a') - \mathbb{E}_{a'|s'} v(s', a')\right)\right| \nonumber \\
            % &\leq \gamma \cdot \sum_{s'} p(s'|s, a) \underset{s', a'}{\max}\left|u(s', a') - v(s', a')\right| \nonumber \\
            &\leq \gamma \cdot \underset{s', a'}{\max}\left|u(s', a') - v(s', a')\right| \nonumber \\
            \Rightarrow \underset{s, a}{\max}\left|\mathcal{T}^\pi_Q u(s, a) - \mathcal{T}^\pi_Q v(s, a)\right| &\leq \gamma \cdot \underset{s, a}{\max}\left|u(s, a) - v(s, a)\right| \nonumber \\
            \Rightarrow \|\mathcal{T}^\pi_Q u - \mathcal{T}^\pi_Q v\|_\infty &\leq \gamma \cdot \|u - v\|_\infty
        \end{align}

        This proves that $\mathcal{T}^\pi_Q$ is a contraction mapping with a factor of $\gamma$.

        \item \textbf{Monotonicity:}

        Let us assume that $u(s, a) \leq v(s, a)$ for all $(s, a)$. Then:
        \begin{align}
            \mathcal{T}^\pi_Q u(s, a) - \mathcal{T}^\pi_Q v(s, a) &= \gamma \cdot \sum_{s'} p(s'|s, a) \left(\mathbb{E}_{a'|s'} u(s', a') - \mathbb{E}_{a'|s'} v(s', a')\right) \nonumber \\
            % &= \gamma \cdot \sum_{s'} p(s'|s, a) \cdot \underset{u(s', a') \leq v(s', a') \forall s', a'}{\underbrace{\left(\mathbb{E}_{a'|s'} \left(u(s', a') - v(s', a')\right)\right)}} \nonumber \\
            % &\leq 0 \nonumber \\
            \Rightarrow \mathcal{T}^\pi_Q u(s, a) - \mathcal{T}^\pi_Q v(s, a) &\leq 0 \nonumber \\
            \Rightarrow \mathcal{T}^\pi_Q u(s, a) &\leq \mathcal{T}^\pi_Q v(s, a)
        \end{align}

        This proves that $\mathcal{T}^\pi_Q$ is a monotonic mapping.
    \end{enumerate}
\end{proof}
Therefore, $\mathcal{T}^\pi_Q$ , as defined, allows us to compute the action-value function associated with the policy $\pi$, and the solution is unique.

\subsection{Policy Evaluation and Iteration in Operators Setting}
For simplicity, we will now present the components of the policy iteration algorithm using these operators. Despite its simplicity, this algorithm is the core of value-based methods.
\begin{enumerate}
    \item \textbf{Policy Evaluation}: Given a fixed policy $\pi$, we iteratively update the value function $v$ by applying the Bellman operator for the policy $\pi$:
    \[
    v_{k+1} \leftarrow \mathcal{T}^{\pi} v_k.
    \]
    By the Banach Fixed-Point Principle, this sequence converges to the value function $v_{\pi}$ as $k \rightarrow \infty$.

    \item \textbf{Policy Iteration}: Start with an initial policy $\pi_0$ and alternate between two steps:
    \begin{itemize}
        \item \textbf{Policy Evaluation}: Just as defined above. Compute the value function for the current policy:
        \[
        v_{k+1} \leftarrow \mathcal{T}^{\pi_i} v_k.
        \]
        \item \textbf{Policy Improvement}: Update the policy by choosing actions that maximize the value:
        \[
        \pi_{i+1}(s) = \underset{a \in \mathcal{A}(s)}{\text{argmax}}~ q_{\pi_i}(s,a).
        \]
    \end{itemize}
    This iterative process converges to the optimal policy $\pi^*$ and the corresponding value function $v_{\pi^*}$ as $i \rightarrow \infty$.
\end{enumerate}

\section{Alternatives Bellman Operator}\label{section4}

In this section, we discuss the limitations of the classical Bellman operators as introduced in the previous sections. We explore the inherent trade-off between achieving optimality and maintaining efficiency, and we present experimental results that highlight the need for refinements, either to the Bellman operators themselves or to the associated value functions. As we have seen, value-based reinforcement learning algorithms solve decision-making problems through the iterative application of a convergent operator, which recursively improves an initial value function. 

\vspace{0.5cm}

\noindent
While the classical Bellman operator has been widely used in reinforcement learning, numerous studies have proposed alternatives to address its limitations \cite{asadi2017alternative, azar2011speedy, bellemare2016increasing, bertsekas2012q, lu2018general, watkins1989learning}. Among these alternatives, two approaches have shown particularly promising results: the \textbf{consistent Bellman Operator} \cite{bellemare2016increasing} and the family of \textbf{Robust Stochastic Operators} \cite{lu2018general}. These operators offer improvements in different aspects compared to the standard Bellman operator, which motivates us to examine them more closely. The first alternative, the \textbf{consistent Bellman Operator}, introduces a modification that better aligns the learned value function with the underlying policy, thus improving the performance of value-based methods in practice. The second alternative, the \textbf{Robust Stochastic Operators}, generalize the Bellman operator to provide robustness against uncertainty and variability in the environment, offering enhanced stability during learning. Although these operators, especially the Robust Stochastic Operator, show significant promise, \textbf{we suggest a non-stochastic and refined version that could potentially improve its performance. This refinement aims to increase stability without relying on the stochastic nature of the operator, making it more broadly applicable and effective}. In the following sections, we will dive deeper into the mathematical formulation of these alternative operators and present the results of our experiments. These experiments demonstrate that both the consistent Bellman operator and our proposed refinement of the Robust Stochastic Operator yield improvements in stability and convergence speed. These findings suggest that further exploration of these alternative operators, and possibly others, is a fruitful direction for enhancing the performance of reinforcement learning algorithms.

%In this section, we discuss some limitations of the classical Bellman operators as introduced in the previous parts. We present how we have to make a trade-off between optimality and efficiency, and finally we show some practical results that we have obtained through experiments, which suggest that we need to think more about some refinements that should be made either to the Bellman operators or to the associated value functions. As we have seen, value-based reinforcement learning algorithms solve decision-making problems through iterative application of a convergent operator, and an initial value function is recursively improved. Many other research works have examined alternatives to the Bellman operator \cite{asadi2017alternative,azar2011speedy,bellemare2016increasing,bertsekas2012q,lu2018general,watkins1989learning}, but there are two which seem to be more promising; this is the reason why we will take inspiration from them. The first one is the \textbf{consistent Bellman Operator} \cite{bellemare2016increasing}, and the second one is the family of \textbf{Robust Stochastic Operators} \cite{lu2018general}. These two operators appear to be more interesting than those suggested before \cite{lu2018general}, but for the second, \textbf{we'll suggest a refinement of it which will remain good (we postulate) even without the stochasticity, improving stability}.

\subsection{Motivation \cite{bellemare2016increasing,lu2018general}}

The motivation for exploring new formulations of the Bellman Operator is clear. While Q-learning and other value-based methods have been successfully applied in reinforcement learning (RL) to find optimal policies, there remains a constant need to improve their convergence speed, accuracy, and robustness. A critical factor in this regard is the presence of intrinsic approximation errors, which arise frequently in real-world scenarios. For instance, when using a discrete Markov Decision Process (MDP) to approximate a continuous system, the value function obtained through the Bellman operator may not accurately represent the value of stationary policies. More importantly, when the differences between the optimal state-action value function and suboptimal value functions are small, these minor discrepancies can lead to errors in identifying the truly optimal actions. This issue becomes even more pronounced in environments where approximations are necessary, as is often the case when continuous-time systems are discretized. In such situations, the Bellman operator may not generalize well, and errors in value estimation can propagate through the learning process, leading to suboptimal performance. Thus, while classical Bellman Operators perform well in perfectly discrete settings, we must refine them to be more generalizable to practical, real-world problems where intrinsic errors are unavoidable. For these types of problems, which often arise from discretizing continuous systems, there is always an inherent approximation error. To address this, it is essential to integrate a corrector mechanism into the operator to account for these discrepancies and improve its performance across different settings. In the following, we will explore the effectiveness of the two alternative operators mentioned earlier: the consistent Bellman Operator and our modified version of the Robust Stochastic Operator. By analyzing their performance, we aim to gain further insights into how these refinements can help in the general approach to determining optimal policies in reinforcement learning.

\subsection{The Consistent Bellman Operator}
The consistent Bellman operator was mentioned for the first time in \cite{bellemare2016increasing} for the action-value function. It's defined as follows:
\begin{equation}\label{ConsistentBellman}
	\mathcal{T}_c f(s, a) = r(s, a) + \gamma \cdot \sum_{s'} p(s' | s, a) \cdot \left[\mathbb{I}_{\{s \neq s'\}} \underset{a'}{\max} f(s', a') + \mathbb{I}_{\{s = s'\}} f(s, a)\right], \quad \text{with } f \in \mathcal{Q},
\end{equation}
where \( \mathbb{I} \) denotes the indicator function. And from section \ref{Section3.3} remember the $\mathcal{Q}$ is the space of action-value functions (space of bounded real-valued functions over $\mathcal{S}\times\mathcal{A}$).

We claim that the consistent Bellman operator given by Equation \ref{ConsistentBellman} satisfies the following important properties:
\begin{enumerate}
	\item \( \mathcal{T}_c \) is a contraction mapping.
	\item \( \mathcal{T}_c \) is monotonic.
\end{enumerate}

\begin{proof}
	To streamline the proof, we first make some refinements:
	\begin{itemize}
		\item For simplicity, we rewrite the expectation over the transition probabilities:
		\[
		\sum_{s'} p(s' | s, a) f \quad \text{as} \quad \mathbb{E}_\mathbb{P}(f).
		\]
		\item We also rewrite the action-value function as follows:
		\begin{equation}\label{QRefinement}
			f_s(s', a') = 
			\begin{cases}
				f(s', a'), & \text{if } s \neq s', \\
				f(s, a), & \text{if } s = s',
			\end{cases}
			\quad \text{for } f \in \mathcal{Q}.
		\end{equation}
	\end{itemize}

	With those refinements, we can rewrite the consistent Bellman operator from Equation \ref{ConsistentBellman} like this:
	\begin{equation}\label{ConsistentBellmanRef}
		\mathcal{T}_c f(s, a) = r(s, a) + \gamma \cdot \mathbb{E}_\mathbb{P}\left[\max_{a'} f_s(s', a')\right], \quad \text{with } f \in \mathcal{Q}.
	\end{equation}

	Now, we proceed with the proofs:
	\begin{enumerate}
		\item \textbf{Contraction:} Let \( u, v \in \mathcal{Q} \). We need to show that \( \mathcal{T}_c \) is a contraction:
		\begin{align}
			\left|\mathcal{T}_c u(s, a) - \mathcal{T}_c v(s, a)\right| &= \left| \gamma \cdot \mathbb{E}_\mathbb{P}\left(\underset{a'}{\max}~ u_s(s', a')\right) - \gamma \cdot \mathbb{E}_\mathbb{P}\left(\underset{a'}{\max}~ v_s(s', a')\right) \right| \nonumber \\
			&\leq \gamma \cdot \left| \underset{a'}{\max}~ \mathbb{E}_\mathbb{P}\Big(u_s(s', a') - v_s(s', a')\Big) \right| \nonumber \\
			&\leq \gamma \cdot \underset{s', a'}{\max} \left| u_s(s', a') - v_s(s', a') \right| \nonumber \\
			&= \gamma \cdot \underset{s, a}{\max} \left| u_s(s, a) - v_s(s, a) \right| \nonumber \\
			&= \gamma \cdot ||u_s(s, a) - v_s(s, a)||_\infty, \nonumber
		\end{align}
		which shows that:
		\[
		||\mathcal{T}_c u - \mathcal{T}_c v||_\infty \leq \gamma \cdot ||u - v||_\infty.
		\]
		Hence, \( \mathcal{T}_c \) is a contraction.
		
		\item \textbf{Monotonicity:} Consider two state-action value functions \( u \) and \( v \) such that \( u(s, a) \leq v(s, a) \) for all \( (s, a) \in \mathcal{S} \times \mathcal{A} \). From this, we have \( u_s(s, a) \leq v_s(s, a) \). Now, we show that:
		\begin{align}
			\mathcal{T}_c u(s, a) - \mathcal{T}_c v(s, a) &\leq \gamma \cdot \underset{a'}{\max} \left( \mathbb{E}_\mathbb{P}\left[ u_s(s', a') - v_s(s', a') \right] \right) \nonumber \\
			&\leq 0, \nonumber
		\end{align}
		which implies:
		\[
		\mathcal{T}_c u(s, a) \leq \mathcal{T}_c v(s, a), \quad \forall u, v \in \mathcal{Q}.
		\]
		Thus, \( \mathcal{T}_c \) is monotonic.
	\end{enumerate}
\end{proof}

\noindent
From this proof, we can conclude that \( \mathcal{T}_c \) has a unique fixed point, and this fixed point corresponds to the optimal value function associated to the consistent Bellman equation instead of the classical one.

\noindent
\textbf{The key question that remains is how this fixed point relates to the one obtained using the classical Bellman operator.} While we are confident that both operators lead to unique fixed point, it is not immediately clear how they compare. At this stage, without the appropriate mathematical tools to analyze the relationship between the two fixed points, we will rely on empirical results to observe how the consistent Bellman operator behaves in practice compared to the classical Bellman operator.

\subsection{Modified Robust Stochastic Operator}
Our proposed operator takes inspiration from both \cite{lu2018general} and \cite{bellemare2016increasing}, with a greater emphasis on the approach taken in the first. Our refinement is more general (an expectation operator, with refined concepts) and differs from the Robust Stochastic Operator suggested in \cite{lu2018general}, showing the irrelevance of stochasticity if the concepts are defined accordingly.

\vspace{0.5cm}

\noindent
Naturally, we can express \( v_\pi(s) \) as:
\[
v_\pi(s) = \sum_{a}\pi(a|s)\cdot q_\pi(s,a),
\]
and the difference between the action value function and the state value funCtion, known as \textbf{advantage learning}, is given by:
\[
A(s,a) = q_\pi(s,a) - v_\pi(s),
\]
which provides insights into the quality of the policy, as well as the value function. During the learning process of the optimal policy, as long as the chosen algorithm improves the policy, the quantity \( \left|A(s,a)\right| \) should decrease, indicating better action selection.

We propose modifying the \textbf{Bellman Expectation Operator for the action-value function}, as defined in Equation \ref{eq:BellExpeOperQValues}, by directly integrating the concept of \textbf{advantage learning} into the operator, instead of applying it later in the learning process, as is sometimes implicitly done in implementations of policy gradient methods \cite{graesser2019foundations}. Let this new operator be denoted as \( \mathcal{T}_a \), defined for all \( f \in \mathcal{Q} \) as:
\begin{equation}\label{AdvantageOperator}
	\left(\mathcal{T}_a f\right)(s,a) = r(s,a) + \gamma \cdot \sum_{s'} p(s' | s, a) \left( \sum_{a'} \pi(a' | s') \cdot f(s', a') \right) + \beta \cdot \underbrace{\left[ f(s, a) - \sum_{a} \pi(a | s) f(s, a) \right]}_{\text{advantage learning}}.
\end{equation}
We will now examine the properties of the operator defined by Equation \ref{AdvantageOperator}. The coefficient \( \beta \) is currently any real number, but we will define it appropriately by the end of this theoretical discussion.

\begin{pro}
	The operator \( \mathcal{T}_a \), as defined in Equation \ref{AdvantageOperator}, is not a contraction mapping.
\end{pro}

\begin{proof}
	Let \( u(s,a) \equiv u \) and \( v(s,a) \equiv v \) be two elements of \( \mathcal{Q} \):
	\begin{align}
		\left|\mathcal{T}_a u(s,a) - \mathcal{T}_a v(s,a)\right| &= \left| \gamma \cdot \mathbb{E}_\mathbb{P}\left( \sum_{a'} \pi(a' | s') \left( u(s', a') - v(s', a') \right) \right) + \beta \cdot \left( \left[ u - v \right] - \sum_{a} \pi(a | s) \left[ u(s,a) - v(s,a) \right] \right) \right| \nonumber \\
		&> \beta \cdot \left| u(s,a) - v(s,a) \right| \quad \text{for certain \( (u, v) \) and \( \beta \)} \nonumber \\
		\Rightarrow ||\mathcal{T}_a u - \mathcal{T}_a v||_\infty &> \beta \cdot ||u - v||_\infty\nonumber \\
		\Rightarrow ||\mathcal{T}_a u - \mathcal{T}_a v||_\infty &> \gamma \cdot ||u - v||_\infty  \quad \text{for certain \( (u, v) \) and since \( \beta\approx\gamma \) at some point (see proposition \ref{PropBeta}))}.
	\end{align}
	This simplified proof shows that the operator \( \mathcal{T}_a \) is usually not a contraction mapping.
\end{proof}

Since \( \mathcal{T}_a \) is not a contraction, we cannot guarantee convergence. However, despite this, we can still examine the operator's behavior. Drawing from \cite{lu2018general} and \cite{bellemare2016increasing}, we introduce two key properties by which we will qualify a value-based operator as \textbf{well-behaving operator}: \textit{optimality preservation} and \textit{gap increasing}.\\
$_{}$\\
\textbf{Note}: In violation of our notation conventions, for simplification, below we are going to write value functions using capital letters, which will also help distinguishing operators written as indices.

\begin{defn}[Optimality Preservation]\label{OptimPreserv}
	Let \( \mathcal{T}_a \) be an alternative operator to the Bellman operator \( \mathcal{T}_b \). We say that \( \mathcal{T}_a \) preserves optimality if:
	\[
	Q_{k,\mathcal{T}_b} < V_{k,\mathcal{T}_b} \implies Q_{k,\mathcal{T}_a} < V_{k,\mathcal{T}_a} \quad \text{as } k \to \infty,
	\]
	where \( Q_{k,\mathcal{T}} \) represents the action-value function at iteration \( k \) using operator \( \mathcal{T} \), and \( V_{k,\mathcal{T}} \) is the associated state-value function.
\end{defn}

\begin{defn}[Gap Increasing]\label{GapIncreasing}
	Using the same notation as before, we say that an alternative operator \( \mathcal{T}_a \) induces the gap increasing property if, for every state \( s \in \mathcal{S} \) and each feasible action \( a \in \mathcal{A}(s) \):
	\[
	\left| \lim_{k \to \infty} \left( Q_{k,\mathcal{T}_b} - V_{k,\mathcal{T}_b} \right) \right| \leq \left| \lim_{k \to \infty} \left( Q_{k,\mathcal{T}_a} - V_{k,\mathcal{T}_a} \right) \right|.
	\]
\end{defn}

\noindent
The \textbf{optimality preservation} property indicates how well the operator preserves the search for the optimal fixed point, while the \textbf{gap increasing} property reflects the ability of the operator to distinguish between the values of suboptimal and optimal actions.

\begin{pro}
	We claim that the operator defined in Equation \ref{AdvantageOperator} is a well-behaving operator, even though it is not a contraction.
\end{pro}

\begin{proof}
	We prove these properties in sequence:
	\begin{enumerate}
		\item \textbf{Optimality Preservation}: From Definition \ref{OptimPreserv}, we assume that \( Q_{k,\mathcal{T}_b} < V_{k,\mathcal{T}_b} \), and aim to show \( Q_{k,\mathcal{T}_a} < V_{k,\mathcal{T}_a} \) as \( k \to \infty \). Note that for better flow, when applying the Bellman operator to a quantity, we will sometimes write \( \text{Bell}(Q_{k-1}) \) instead of \( \mathcal{T}_b Q_{k-1} \). We start with:
		{\begin{align}
			Q_{k,\mathcal{T}_b}< V_{k,\mathcal{T}_b}& \Rightarrow \mathcal{T}_b Q_{k-1}< V_{k,\mathcal{T}_b} \nonumber \\
			&\Rightarrow Bell(Q_{k-1}) < \sum_{a} \pi(a|s) Bell(Q_{k-1}). \nonumber
		\end{align}
		Without loss of generality, let \( K = \min\limits_a\, \Big(Q_{k-1}(s,a) - V_{k-1}(s)\Big) \). After including the advantage learning term, we find:
		\begin{align}
			\Rightarrow Bell\left(Q_{k-1}\right) + \beta\cdot K &< \sum_{a}\pi(a|s)\cdot Bell\left(Q_{k-1}\right)+\beta\cdot K\cdot1\nonumber\\
			\Rightarrow Bell\left(Q_{k-1}\right) + \beta\cdot K &< \sum_{a}\pi(a|s)\cdot Bell\left(Q_{k-1}\right)+\beta\cdot\sum_{a}\pi(a|s)\cdot K\nonumber	
		\end{align}}
        also knowing that $K\leq\sum_{a}\pi(a|s)\cdot\left(Q_{k-1}(s,a)-V_{k-1}(s)\right) ~~\forall~ (s,a)  \in \mathcal{S}\times\mathcal{A}$, and since we are getting better policy as $k\to\infty$, we have almost surely that $\left(Q_{k-1}-V_{k-1}\right)\leq \sum_{a}\pi(a|s)\cdot\left(Q_{k-1}-V_{k-1}\right)$. Therefore:
            \begin{align}
                Bell\left(Q_{k-1}\right) + \beta\cdot K &< \sum_{a}\pi(a|s)\cdot Bell\left(Q_{k-1}\right)+\beta\cdot\sum_{a}\pi(a|s)\cdot K\nonumber\\
			\Rightarrow Bell\left(Q_{k-1}\right) + \beta\cdot \left(Q_{k-1}-V_{k-1}\right) &< \sum_{a}\pi(a|s)\cdot Bell\left(Q_{k-1}\right)+\beta\cdot\sum_{a}\pi(a|s)\cdot \left(Q_{k-1}-V_{k-1}\right) ~~~~~\text{as $k\to\infty$}\nonumber\\
			\Rightarrow \underset{Q_{k,\mathcal{T}_a}}{\underbrace{Bell\left(Q_{k-1}\right) + \beta\cdot \left(Q_{k-1}-V_{k-1}\right)}} &< \sum_{a}\pi(a|s)\cdot\Big[ Bell\left(Q_{k-1}\right)+\beta\cdot \left(Q_{k-1}-V_{k-1}\right)\Big]\nonumber\\
			\Rightarrow Q_{k,\mathcal{T}_a} &< \sum_{a}\pi(a|s)\cdot Q_{k,\mathcal{T}_a}(s,a)\nonumber\\	
			\Rightarrow Q_{k,\mathcal{T}_a} & < V_{k,\mathcal{T}_a}	~~~~~~~~~\text{as $k\to\infty$}					
		\end{align}
		\item \textbf{Gap Increasing}: Using Definition \ref{GapIncreasing}, we need to show that:
		\[
		\left| \lim_{k \to \infty} \left( Q_{k,\mathcal{T}_b} - V_{k,\mathcal{T}_b} \right) \right| \leq \left| \lim_{k \to \infty} \left( Q_{k,\mathcal{T}_a} - V_{k,\mathcal{T}_a} \right) \right|.
		\]
        Recall that for optimal actions,
		\[\lim\limits_{k\to\infty} \left[Q_{k,\mathcal{T}_b} - V_{k,\mathcal{T}_b}\right] = 0,\]
		Therefore, using $\mathcal{T}_a$, we can assert, for optimal actions, that:
		\[\left|\lim\limits_{k\to\infty}\left[Q_{k,\mathcal{T}_b} - V_{k,\mathcal{T}_b}\right]\right| = 0 ~~~~\leq~~~~ \left|\lim\limits_{k\to\infty} \left[Q_{k,\mathcal{T}_a} - V_{k,\mathcal{T}_a}\right]\right|\]
  For all other actions, since convergence implies \(\,Q_{k-1}-V_{k-1}<0\) and optimality is preserved (previous proof), the same inequality holds. Therefore,
  \[
		\left| \lim_{k \to \infty} \left( Q_{k,\mathcal{T}_b} - V_{k,\mathcal{T}_b} \right) \right| \leq \left| \lim_{k \to \infty} \left( Q_{k,\mathcal{T}_a} - V_{k,\mathcal{T}_a} \right) \right|~~~~~~~\forall (s,a)\in \mathcal{S}\times\mathcal{A}.
		\]
		Thus, for any state \( s \) and action \( a \), the operator \( \mathcal{T}_a \) maintains or increases the gap compared to \( \mathcal{T}_b \), satisfying the gap increasing property as we defined it.
	\end{enumerate}
\end{proof}

\noindent
The operator defined in Equation \ref{AdvantageOperator} is a \textbf{well-behaving operator}, even though we assume convergence under specific conditions for \( \beta \).

\vspace{0.5cm}

\noindent
Now, we discuss the possibility of finding a fixed point for this operator. Before doing so, we will first examine the continuity and boundedness of the operator, as these are prerequisites for analyzing fixed points.

\begin{pro}\label{PropBeta}
	Let the reward function \( r(s,a) \) be bounded and continuous. The operator \( \mathcal{T}_a \), as defined in Equation \ref{AdvantageOperator}, is also \textbf{bounded} and \textbf{continuous}.
\end{pro}

\begin{proof}
	We can break the operator into three components:
	\[
	\left(\mathcal{T}_a f\right)(s,a) = \underbrace{r(s,a)}_{\text{Part 1}} + \underbrace{\gamma \cdot \sum_{s'} p(s' | s, a) \left( \sum_{a'} \pi(a' | s') \cdot f(s', a') \right)}_{\text{Part 2}} + \underbrace{\beta \cdot \left[ f(s,a) - \sum_{a} \pi(a | s) f(s,a) \right]}_{\text{Part 3}}.
	\]
	\begin{enumerate}
		\item \textbf{Part 1}: The reward function is bounded and continuous by assumption.
		\item \textbf{Part 2}: This part is a contraction mapping, hence it is both bounded and continuous.
		\item \textbf{Part 3}: If \( f \) is bounded and continuous, then the term \( f(s,a) - \sum_{a} \pi(a | s) f(s,a) \) is also bounded and continuous, provided \( \beta \) has appropriate properties.
	\end{enumerate}
	Thus, \( \mathcal{T}_a \) is continuous and bounded, assuming the behavior of \( \beta \) ensures convergence.
\end{proof}

Given the boundedness and continuity of \( \mathcal{T}_a \), we expect that the operator will not diverge if \( \beta \) is chosen carefully. The choice of \( \beta \) must balance between improving speed and optimality while maintaining proximity to the classical Bellman Operator. 
% Finally, we suggest conditions for \( \beta \) to ensure convergence. Consider a family of operators based on \( \mathcal{T}_a \), with \( \beta \) varying across iterations. For convergence, the sequence \( \beta_{i,j} \) should satisfy:
% \begin{equation}\label{eq:ConditionsBeta}
% 	\left\{
% 	\begin{array}{lcl}
% 		\sum_{j=1}^{\infty} \beta_{i,j} &<& \infty, \\
%         &&\\
% 		\{\beta_{i,j}\} &\to& 0, \quad \text{as } j \to \infty.
% 	\end{array}
% 	\right.
% \end{equation}
We therefore propose conditions for \( \beta \) to ensure convergence in a family of operators based on \( \mathcal{T}_a \), where \( \beta \) varies across iterations \( j \) with \( i \) refering to the index of a specific operator within this family. For convergence, the sequence \( \beta_{i,j} \) must satisfy the following two conditions:

% \begin{equation}\label{eq:ConditionsBeta}
% 	\left\{
% 	\begin{array}{lcl}
% 		\sum_{j=1}^{\infty} \beta_{i,j} &<& \infty, \\
% 		&&\\
% 		\{\beta_{i,j}\} &\to& 0, \quad \text{as } j \to \infty.
% 	\end{array}
% 	\right.
% \end{equation}
{\begin{equation}
    \sum_{j=1}^{\infty} \beta_{i,j} < \infty\,,  \text{and } \{\beta_{i,j}\}\to 0, \text{ as } j \to \infty.
\end{equation}}

The first condition ensures that the total sum of \( \beta_{i,j} \) across iterations is finite, and the second condition guarantees that \( \beta_{i,j} \) approaches zero as the iteration index \( j \) grows. Together, these conditions allow for the construction of a sequence of operators that converges to the classical Bellman operator.

\section{Implementations and analysis}\label{section5}
To prove the effectiveness of the suggested \textit{Modified Robust Stochastic Operator}, we have conducted experiments on three groups of classical problems in reinforcement learning using \textbf{OpenAI Gymnasium environments}. The illustrative experiments were conducted using Q-Learning, implemented with inspiration from \cite{GithubVmayoral}.
Our own implementation is available in the GitHub repository \cite{MyGithub}.
Interested readers can see all the parameters at the beginning of the code for reproducibility.

\vspace{0.5cm}

\noindent
Bellow  are the necessary details on the environments we used, along with a succinct presentation of the main results we obtained. The goal is to show that refining the operators isn't just a theoretical curiosity but can lead to actual improvements in the algorithm's performance. We should also mention that, although it is not part of this paper, we also considered the function-approximation case (using DQN), which provided significant improvements as well; however, that investigation should be part of an extensive study.

\subsection{Environments and results}
\begin{itemize}
	\item[1.] \textbf{Mountain Car environment:} The theory about this environment is presented in \cite{moore1990efficient}. The state vector is 2-dimensional, continuous with a total of three possible actions. As long as the goal is not yet reached, depending on the action, a negative reward is given to the agent until it reaches the goal. Following \cite{lu2018general}, we have discretized the state space into a \( 40\times 40 \) grid, but differently, we did 10,000 training steps, with 10,000 episodes each.\\
    The following Figure \ref{fig:mountaincar} shows the average total reward across episodes, and we can see that our modified version of the Robust Stochastic Bellman Operator achieves better performance. We will provide a comparative analysis of all the results together in the next section.
    \begin{center}
        \includegraphics[width=0.8\linewidth]{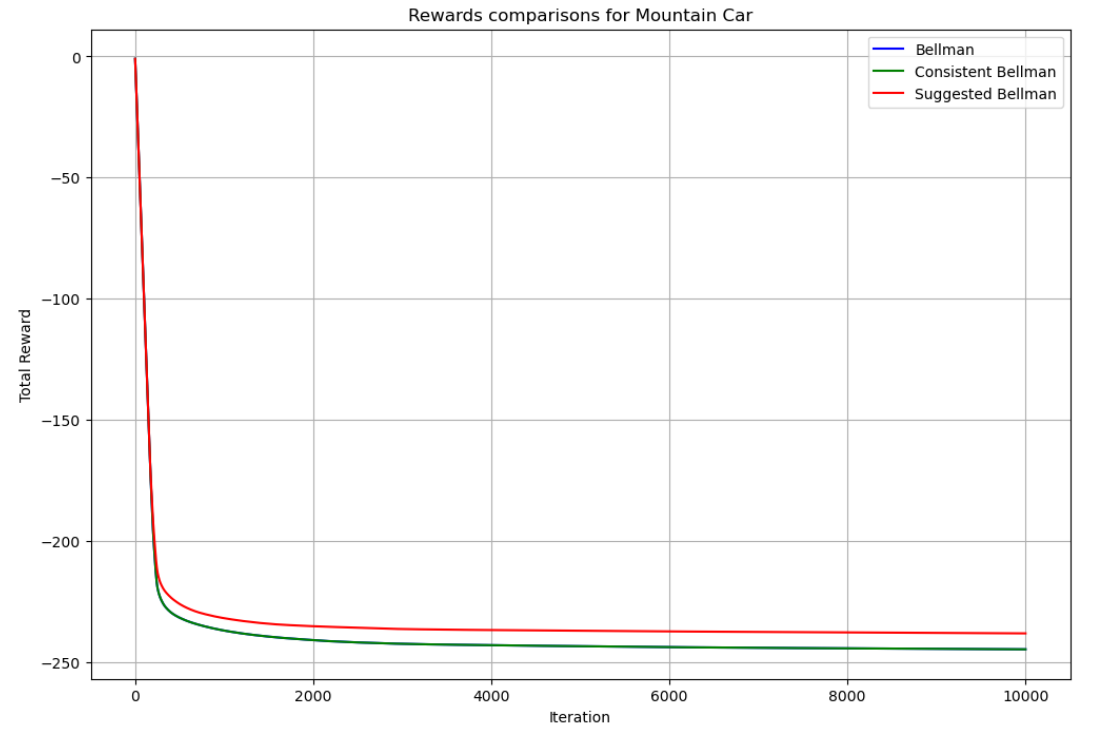}
        \captionof{figure}{Convergence comparison in the MountainCar environment.}
        \label{fig:mountaincar}
    \end{center}

	\item[2.] \textbf{Cart Pole environment:} The theory about Cart Pole is presented in \cite{barto1983neuronlike}. For this environment, the state vector is 4-dimensional and continuous with a total of two possible actions. The aim is to keep the pole upright for as long as possible, with a reward of \( +1 \) for each step up to the failure, including the final step. So, the reward is positive at the end.\\
    During our experiments, we have discretized the state space into a \( 150\times 150\times 150 \times 150 \) grid and again we did 10,000 training steps, with 10,000 episodes each. The following Figure \ref{fig:cartpole} shows the average total rewards across episodes. And again, as predicted, there is a significant improvement when using the modified operator.
    \begin{center}
        \includegraphics[width=0.8\linewidth]{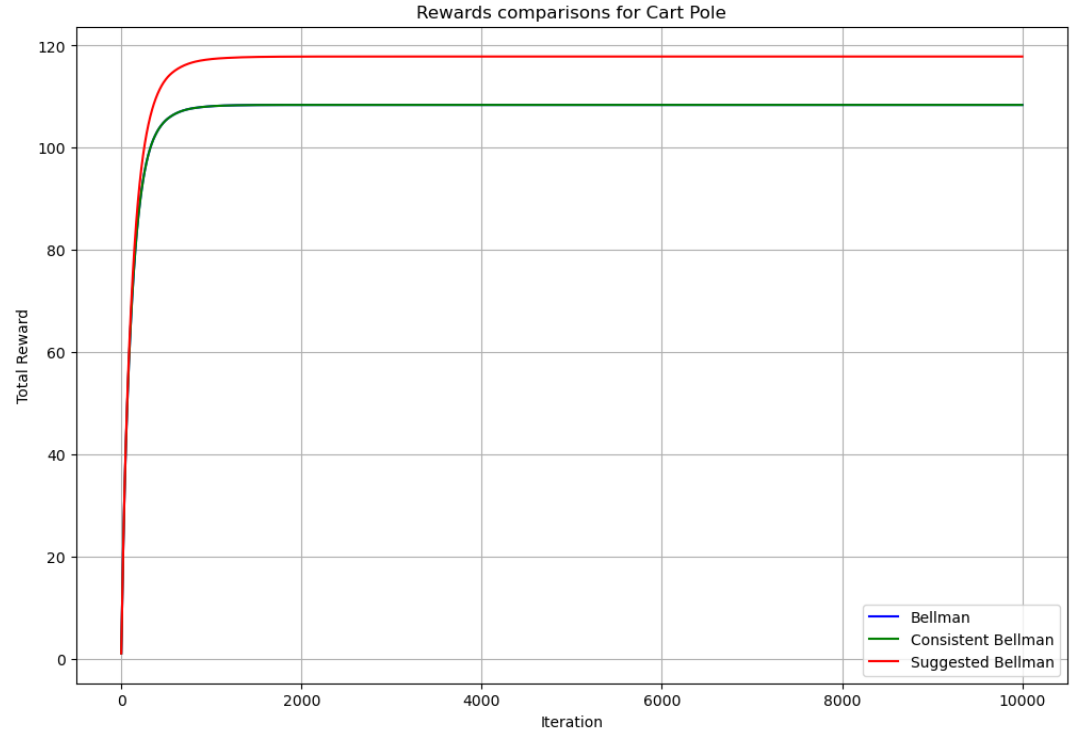}
        \captionof{figure}{Performance in the CartPole environment.}
        \label{fig:cartpole}
    \end{center}

	\item[3.] \textbf{Acrobot environment:} The theory of the Acrobot environment and all its specifications is presented in \cite{sutton1995generalization}. The state vector is 6-dimensional and continuous with a total of three possible actions. The goal in this environment is to have the free end of the Acrobot reach the target height (represented by a horizontal line) in as few steps as possible, with each step not reaching the target being rewarded with -1. So, the reward is negative again as for Mountain Car.
    During our experiments, we have discretized the state space into a \( 30\times 30\times 30 \times 30\times 30 \times 30 \) grid, due to the limitations in memory allocation of the computer we were using, and again we did 10,000 training steps, with 10,000 episodes each. The following Figure \ref{fig:acrobot} shows the averages across episodes.
    \begin{center}
        \includegraphics[width=0.8\linewidth]{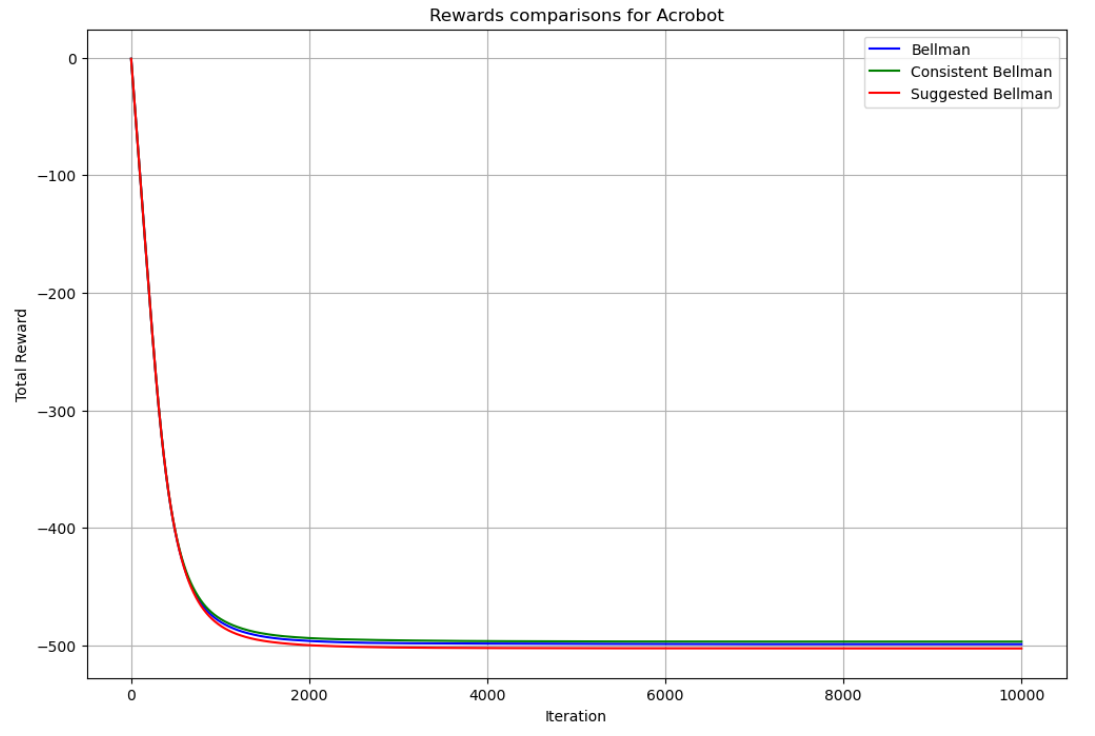}
        \captionof{figure}{Comparison of learning curves in the Acrobot environment.}
        \label{fig:acrobot}
    \end{center}
\end{itemize}

\subsection{Interpretation}
\begin{itemize}
	\item[1.] \textbf{Mountain Car:} In the Mountain Car task, the cumulative reward is negative and driven by the number of time-steps the agent takes to reach the goal. In Figure \ref{fig:mountaincar}, the red curve (Modified Robust Stochastic Bellman Operator) clearly converges to a higher (i.e.,less negative) average reward than both the classical Bellman operator (blue) and the Consistent Bellman operator (green), which here are essentially indistinguishable. This indicates that our modification not only accelerates convergence but also reduces the time penalty more effectively, allowing the agent to reach the goal in fewer steps on average.

	\item[2.] \textbf{Cart Pole:} For Cart Pole, the objective is to maximize the number of steps before failure, so higher rewards correspond to longer pole-balancing episodes. As shown in Figure \ref{fig:cartpole}, our Modified Robust Bellman operator (red) again outperforms the other two operators by reaching a superior reward plateau (around 120) significantly faster. The gap between the red curve and the overlapping blue/green curves illustrates a substantial improvement in sample efficiency and performances: the modified operator requires fewer iterations to achieve near-optimal performance.

	\item[3.] \textbf{Acrobot:} The Acrobot problem is more challenging due to its higher dimensionality and sparser reward structure. In Figure \ref{fig:acrobot}, all three curves-classical, consistent, and Modified Robust converge to nearly the same average reward. We believe this is a consequence of the relatively coarse state discretization we were forced to use (owing to memory constraints), which limits the expressiveness of our value estimates. While our modified operator does not yet show a clear advantage under these discretization settings, it nevertheless matches the performance of the baseline methods. A finer discretization or a function-approximation approach (e.g.,DQN) may reveal the full potential of the modified operator in future work. So, this experiment needs further investigations to establish clearly what is the great attainable difference between those operators.
\end{itemize}

\vspace{0.5cm}

\noindent
So, generally the Consistent Bellman Operator yields results virtually identical to those of the classical operator, whereas our Modified Robust Stochastic Operator consistently outperforms both. These illustrative experiments underscore that the theoretical refinements we developed (namely, the incorporation of advantage learning and a carefully decayed constant during training, in accordance with the conditions we proved) have a clear, practical impact on algorithm performance. While deriving RL methods from the classical Bellman framework remains a sound approach, our findings demonstrate that even slight, theory-guided modifications can unlock significant gains, pointing the way toward more effective operator and algorithm designs in future work.
\subsection{Further Discussions About the Statistical Methodology}
For each environment, we conducted \textbf{10 000 independent training runs}. Each run comprised up to 10 000 timesteps, with episodes capped at 10 000 steps or terminated immediately upon the agent reaching its goal. Within each run, we maintained a running total of the discounted reward after every action. If an episode terminated early, we padded the remaining timesteps with the last recorded total so that every episode yielded a fixed-length sequence. After completing all episodes, we averaged these per-step totals across the 10 000 runs to produce a single, smooth learning curve (see figures \ref{fig:mountaincar}, \ref{fig:cartpole} and \ref{fig:acrobot} above). This Monte Carlo Sampling approach aggregation treats each run as an independent trial and provides an unbiased, low-variance estimate of the agent's expected return over time. We emphasize that this classical approach (common in RL benchmarks and often involving far fewer runs, typically 30 to 100) is statistically rigorous when scaled to 10 000 trials, as it substantially reduces estimator variance.\\
\\
If further statistical analysis is of interest, this remains an appealing avenue to explore. However, our primary goal was to demonstrate that the mathematical analysis we conducted warrants serious consideration by both practitioners and theoreticians in the field.
\section{Conclusion}

This work delves into the topological foundations of Reinforcement Learning, providing a rigorous mathematical framework for studying its core concepts. By focusing on the relationship between Topology and Reinforcement Learning principles, we aim to pave the way for breakthroughs that can facilitate the contributions of mathematicians toward improving Reinforcement Learning algorithms.  We discussed the Consistent Bellman Operator as an alternative to the classical Bellman Operator, demonstrating that it retains critical properties, such as the uniqueness of the optimal fixed point. However, this raised important questions about the relationship between the fixed points of these operators and how they might affect algorithm performance in practice. Additionally, we introduced a deterministic variation of the Robust Stochastic Operator, as well as a refinement and simplification of what can be considered a well-behaving operator, highlighting that stochasticity is not essential for achieving superior results compared to the classical Bellman Operator. Our implementation, using Python and OpenAI Gymnasium environments, showed that the proposed operator outperforms classical methods across a variety of tasks, validating its practical effectiveness. Looking ahead, future research could expand upon our approach by further investigating the state space, action space, and policy space, as well as optimizing the efficiency of the proposed operator. In particular, the large-scale Monte Carlo sampling (10 000 independent trials) provides a low-variance, unbiased estimate of agent performance, however more detailed statistical analyses (e.g., confidence intervals, hypothesis tests, percentile metrics) could enrich the analysis by better capturing variability and enabling further comparisons. Ultimately, our goal has been to highlight both the method and its mathematical foundation as deserving of close attention, encouraging both practitioners and theoreticians to build on this framework. We conjecture that any monotonic contraction mapping incorporating a notion of policy could serve as a viable operator in Reinforcement Learning, provided that the definitions of the associated value functions are adequately refined for this context.\\
This work aims to empower researchers (even those who may not be familiar with the technicalities of the field) to engage with Reinforcement Learning through a mathematically rigorous lens, making the fundamental concepts more accessible and understandable, while also illustrating how this perspective can be leveraged to develop efficient algorithms and further investigate the theoretical foundations.

\bibliographystyle{plain}
\bibliography{main}

\begin{thebibliography}{10}

\bibitem{deepminducl2021}
{DeepMind x UCL | Deep Learning Lecture Series 2021}, 2021.
\newblock Accessed on 13/03/2024.

\bibitem{asadi2017alternative}
Kavosh Asadi and Michael~L Littman.
\newblock An alternative softmax operator for reinforcement learning.
\newblock In {\em International Conference on Machine Learning}, pages
  243--252. PMLR, 2017.

\bibitem{DBLP:journals/corr/abs-1804-07193}
Kavosh Asadi, Dipendra Misra, and Michael~L. Littman.
\newblock Lipschitz continuity in model-based reinforcement learning.
\newblock {\em CoRR}, abs/1804.07193, 2018.

\bibitem{azar2011speedy}
Mohammad~Gheshlaghi Azar, Remi Munos, Mohammad Ghavamzadeh, and Hilbert Kappen.
\newblock Speedy q-learning.
\newblock In {\em Advances in neural information processing systems}, 2011.

\bibitem{barto1983neuronlike}
Andrew~G Barto, Richard~S Sutton, and Charles~W Anderson.
\newblock Neuronlike adaptive elements that can solve difficult learning
  control problems.
\newblock {\em IEEE transactions on systems, man, and cybernetics},
  (5):834--846, 1983.

\bibitem{bellemare2016increasing}
Marc~G Bellemare, Georg Ostrovski, Arthur Guez, Philip Thomas, and R{\'e}mi
  Munos.
\newblock Increasing the action gap: New operators for reinforcement learning.
\newblock In {\em Proceedings of the AAAI Conference on Artificial
  Intelligence}, volume~30, 2016.

\bibitem{bertsekas2012q}
Dimitri~P Bertsekas and Huizhen Yu.
\newblock Q-learning and enhanced policy iteration in discounted dynamic
  programming.
\newblock {\em Mathematics of Operations Research}, 37(1):66--94, 2012.

\bibitem{ValuePolytope}
Robert Dadashi, Adrien~Ali Ta{\"{\i}}ga, Nicolas~Le Roux, Dale Schuurmans, and
  Marc~G. Bellemare.
\newblock The value function polytope in reinforcement learning.
\newblock {\em CoRR}, abs/1901.11524, 2019.

\bibitem{graesser2019foundations}
Laura Graesser and Wah~Loon Keng.
\newblock {\em Foundations of deep reinforcement learning: theory and practice
  in Python}.
\newblock Addison-Wesley Professional, 2019.

\bibitem{kadurha2024topological}
David~Krame Kadurha.
\newblock Topological foundations of reinforcement learning.
\newblock {\em arXiv preprint arXiv:2410.03706}, 2024.

\bibitem{MyGithub}
David Krame.
\newblock Reinforcement learning essay (aims-cameroon).
\newblock \url{https://github.com/DavidKrame/rl-essay-aims-cameroon}, 2023.
\newblock Project GitHub Repository.

\bibitem{MetricAndContinuityInRL}
Charline~Le Lan, Marc~G. Bellemare, and Pablo~Samuel Castro.
\newblock Metrics and continuity in reinforcement learning.
\newblock {\em CoRR}, abs/2102.01514, 2021.

\bibitem{lazaric2013markov}
A~Lazaric.
\newblock Markov decision processes and dynamic programming, 2013.

\bibitem{lu2018general}
Yingdong Lu, Mark~S Squillante, and Chai~Wah Wu.
\newblock A general family of robust stochastic operators for reinforcement
  learning.
\newblock {\em arXiv preprint arXiv:1805.08122}, 2018.

\bibitem{moore1990efficient}
Andrew~William Moore.
\newblock Efficient memory-based learning for robot control.
\newblock Technical report, University of Cambridge, Computer Laboratory, 1990.

\bibitem{sigaud2013markov}
Olivier Sigaud and Olivier Buffet.
\newblock {\em Markov decision processes in artificial intelligence}.
\newblock John Wiley \& Sons, 2013.

\bibitem{DBLP:journals/corr/abs-1905-00475}
Zhao Song and Wen Sun.
\newblock Efficient model-free reinforcement learning in metric spaces.
\newblock {\em CoRR}, abs/1905.00475, 2019.

\bibitem{sutton1995generalization}
Richard~S Sutton.
\newblock Generalization in reinforcement learning: Successful examples using
  sparse coarse coding.
\newblock {\em Advances in neural information processing systems}, 8, 1995.

\bibitem{sutton2018reinforcement}
Richard~S Sutton and Andrew~G Barto.
\newblock {\em Reinforcement learning: An introduction}.
\newblock MIT press, 2018.

\bibitem{fixedpoint}
Anita Tomar and M.C Joshi.
\newblock {\em Fixed Point Theory and its Applications to Real World Problems}.
\newblock Nova Science Publishers, New York, 2021.

\bibitem{GithubVmayoral}
vmayoral.
\newblock {Github : Basic Reinforcement Learning}.
\newblock
  \url{https://github.com/vmayoral/basic\_reinforcement\_learning/tree/master},
  2024.

\bibitem{watkins1989learning}
Christopher John Cornish~Hellaby Watkins.
\newblock Learning from delayed rewards.
\newblock 1989.

\end{thebibliography}

\end{document}